%% file: Capitol_SSPG_v2.tex
\documentclass[a4paper]{article}

\usepackage{type1cm}        
\usepackage{makeidx}         
\usepackage{graphicx}        
\usepackage{multicol}        
\usepackage[bottom]{footmisc}
\usepackage{float}

\usepackage[utf8]{inputenc}
\usepackage{newtxtext}       %

\usepackage{authblk}

\makeindex             
\usepackage{color,verbatim}
\usepackage{amsmath}
\usepackage{amsthm}
\usepackage{amssymb}
\usepackage{times}
\usepackage{epstopdf}
\usepackage[caption=false]{subfig}
\usepackage[numbers,sort&compress]{natbib}
\makeatletter
\def\NAT@def@citea{\def\@citea{\NAT@separator}}%

\newtheorem{thm}{Theorem}
\newtheorem{theorem}[thm]{Theorem}
\newtheorem{remark}[thm]{Remark}
\newtheorem{assumption}[thm]{Assumption}

\newcommand{\rset}{\mathbb{R}}

\usecounter{algorithmenumi}

\providecommand{\norm}[1]{\lVert#1\rVert}

\newcommand{\be}{\begin{equation}}
\newcommand{\ee}{\end{equation}}
\newcommand{\eref}[1]{(\ref{#1})}

\newcommand{\I}{{\cal I}}

\newcommand\pdffig[4][7cm]{
	\begin{figure}[t]
		\centering
		\includegraphics[width=#1]{#2}
		\caption{#3}
		\label{#4}
	\end{figure}
}

\begin{document}

\title{Stochastic Proximal Gradient Algorithm with Minibatches. Application to Large Scale Learning Models}

\author[1]{Andrei Pătrașcu}
\author[1,2]{Ciprian Păduraru}
\author[1,2]{Paul Irofti} 

\affil[1]{Computer Science Department, University of Bucharest, Romania}
\affil[2]{The Research Institute of the University of Bucharest (ICUB), Romania}


\maketitle

\begin{abstract}{
Stochastic optimization lies at the core of most statistical learning models. The recent great development of stochastic algorithmic tools focused significantly onto proximal gradient iterations, in order to find an efficient approach for nonsmooth (composite) population risk functions.
The complexity of finding optimal predictors by minimizing regularized risk is largely understood for simple regularizations such as $\ell_1/\ell_2$ norms.
However, more complex properties desired for the predictor necessitates highly difficult regularizers as used in grouped lasso or graph trend filtering.
In this chapter we develop and analyze minibatch variants of stochastic proximal gradient algorithm for general composite objective functions with stochastic nonsmooth components.
We provide iteration complexity for constant and variable stepsize policies obtaining that, for minibatch size $N$, after $\mathcal{O}(\frac{1}{N\epsilon})$ iterations $\epsilon-$suboptimality is attained in expected quadratic distance to optimal solution.
The numerical tests on $\ell_2-$regularized SVMs and parametric sparse representation problems confirm the theoretical behaviour and surpasses minibatch SGD performance.}
\end{abstract}

\thanks{
C. Paduraru and P. Irofti were supported by a grant of 
the Romanian Ministry of Research and Innovation, CCCDI-UEFISCDI,
project number 17PCCDI/2018 within PNCDI III.}

\section{Introduction}

\noindent Statistical learning from data is strongly linked to optimization of stochastic representation models. The traditional approach consists of learning an optimal hypothesis from a reasonable amount of data and further aim to generalize its decision properties over the entire population. In general, this generalization is achieved by minimizing population risk which, unlike the empirical risk minimization, aims to compute the optimal predictor with the smallest generalization error. Thus, in this paper we consider the following stochastic composite optimization problem:
\begin{align}
\label{problem_intro}
\min\limits_{w \in \rset^n} & \;\;  F(w) := \mathbb{E}_{\xi \in \Omega} [f(w;\xi)] + \mathbb{E}_{\xi \in \Omega} [h(w;\xi)],
\end{align}
where $\xi$ is a random variable associated with probability space $(\mathbb{P},\Omega)$, 
$f(w) := \mathbb{E}_{\xi \in \Omega} [f(w;\xi)]$ is smooth and $h(w) := \mathbb{E}_{\xi \in \Omega} [f(w;\xi)]$ convex and nonsmooth. Most convex learning models in population can be formulated following the structure of \eqref{problem_intro} using a proper decomposition dictated by the nature of prediction loss and regularization. 

\noindent Let $f(\cdot;\xi) \equiv \ell(\cdot;\xi)$ be a convex smooth loss function, such as quadratic loss, and $h \equiv r$ the "simple" convex regularization, such as $\norm{w}_1$, then the resulted model:
\begin{align*}
\min\limits_{w \in \rset^n} & \;\;   \mathbb{E}_{\xi \in \Omega} [\ell(w;\xi)] + r(w).
\end{align*}
has been considered in several previous works \cite{MouBac:11,RosVil:14,NiuRec:11,ShaSin:11}, which analyzed iteration complexity of stochastic proximal gradient algorithms.
Here, the proximal map of $r$ is typically assumed as being computable in closed-form or linear time, as appears in $\ell_1/\ell_2$ Support Vector Machines (SVMs). In order to be able to approach more complicated regularizers, expressed as sum of simple convex terms, as required by  machine learning models such as: group lasso \cite{ZhoKwo:14, ShiLin:15, HaLes:15}, CUR-like factorization \cite{WanWan17_prox}, graph trend filtering \cite{SalBia:17,VarLee:19}, dictionary learning \cite{DumIro:18}, parametric sparse representation \cite{StoIro:19}, one has to be able to handle stochastic optimization problems with stochastic nonsmooth regularizations $r(w) = \mathbb{E}[r(w;\xi)]$. For instance, the grouped lasso regularization $\sum\limits_{j=1}^m \norm{D_j w}_2$ might be expressed as expectation by considerind $r(w;\xi) = \norm{D_{\xi}w}_2$. In this chapter we analyze extensions of stochastic proximal gradient for this type of models.

Nonsmooth  (convex) prediction losses (e.g. hinge loss, absolute value loss, $\epsilon-$insensitive loss) are also coverable by \eqref{problem_intro} through taking $h(\cdot;\xi) = \ell(\cdot;\xi)$.
We will use this approach with $f(w) = \frac{\lambda}{2}\norm{w}^2_2$ for solving hinge-loss-$\ell_2$-SVM model.

\vspace{5pt}

\noindent \textbf{Contributions}.$(i)$ We derive sublinear convergence rates of SPG with minibatches for stochastic composite convex optimization, under strong convexity assumption.

\noindent $(ii)$ Besides the sublinear rates, we provide computational complexity analysis, which takes into account the complexity of each iteration, for stochastic proximal gradient algorithm with minibatches. We obtained $\mathbb{O}\left( \frac{1}{N \epsilon}\right)$ which highlights optimal dependency on the minibatch size $N$ and accuracy $\epsilon$.

\noindent $(iii)$ We confirm empirically our theoretical findings through tests over $\ell_2-$SVMs (with hinge loss) on real data, and parametric sparse representation models on random data.

\vspace{5pt}

\noindent Following our analysis, reductions of the complexity per iteration  using multiple machines/processors, would guarantee direct improvements in the optimal number of iterations. The superiority of distributed variants of SGD schemes for smooth optimization are clear (see \cite{NiuRec:11}), but our results set up the theoretical foundations for distributing the algorithms for the class of proximal gradient algorithms.

\noindent Further, we briefly recall the milestone results from stochastic optimization literature with focus on the complexity of stochastic first-order methods.

\subsection{Previous work}

The natural tendency of using minibatches in order to accelerate stochastic algorithms and to obtain better generalization bounds is not new \cite{FriSch:12,GhaLan:16,WanWan17_minibatch,WanSre:19}. Empirical advantages have been observed in most convex and nonconvex models although clear theoretical complexity reductions with the minibatch size are still under development for structured nonsmooth models. 
Great attention has been given in the last decade to the behaviour of the stochastic gradient descent (SGD) with minibatches, 
see \cite{FriSch:12,GhaLan:16, NemJud:09,MouBac:11,NguNgu:18,Ned:10,Ned:11, GhaLan:16,NedNec:19}. On short, SGD iteration computes the average of gradients on a small number of samples and takes a step in the negative direction.
Although more samples in the minibatch imply smaller variance in the direction and, for moderate minibatches, brings a significant acceleration, recent evidence shows that by increasing minibatch size over certain threshold the acceleration vanishes or deteriorates the training performance \cite{GoyDol:17,HoffHub:17}.

\noindent Since the analysis of SGD naturally requires various smoothness conditions, proper modifications are necessary to attack nonsmooth models. The stochastic proximal point (SPP) algorithm has been recently analyzed using various differentiability assumptions,
see \cite{TouTra:16,RyuBoy:16,Bia:16,PatNec:18,KosNed:13,WanBer:16,AsiDuc:18},
and has shown surprising analytical and empirical performances. 
The works of \cite{WanWan17_minibatch,WanSre:19} analyzed minibatch SPP schemes with variable stepsizes and obtained $\left( \frac{1}{k N}\right)$ convergence rates under proper assumptions.
For strongly convex problems, notice that they require multiple assumptions that we avoid using in our analysis: strong convexity on each stochastic component, knowledge of strong convexity constant and Lipschitz continuity on the objective function. Our analysis is based on strong convexity of the smooth component $f$ and only convexity on the nonsmooth component $h$.


\vspace{5pt}

\noindent A common generalization of SGD and SPP are the stochastic splitting methods.
Splitting first-order schemes received significant attention due to their natural insight and simplicity in contexts where a sum of two components are minimized (see \cite{Nes:13,BecTeb:09}).
Only recently the full stochastic composite models with stochastic regularizers have been properly tackled \cite{SalBia:17}, where almost sure asymptotic convergence is established for a stochastic splitting scheme, where each iteration represents a proximal gradient update using stochastic samples of $f$ and $h$. 
The stochastic splitting schemes are also related to the model-based methods developed in \cite{DavDru:19}. 


\vspace{5pt}

\subsection{Preliminaries and notations}

For $w,v \in \rset^n$ denote the scalar product  $\langle w,v \rangle = w^T v$ and Euclidean norm by $\|w\|=\sqrt{w^T w}$. 
We use notations $\partial h(w;\xi)$ for the subdifferential set
and $g_h(w;\xi)$ for a subgradient  of $h(\cdot;\xi)$ at $w$.
In the differentiable case we use the gradient notation $\nabla f(\cdot;\xi)$. 
We denote the set of optimal solutions with $W^*$ and $w^*$ for any optimal point of \eqref{problem_intro}. 
\begin{assumption} \label{assump_basic}
The central problem \eqref{problem_intro} has nonempty  optimal set $W^*$ and satisfies:

\noindent $(i)$ The function $f(\cdot;\xi)$ has $L_f$-Lipschitz gradient, i.e. there exists $L_f > 0$ such that:
\begin{align*}
\norm{\nabla f(w;\xi) - \nabla f(v;\xi)} \le L_f \norm{w-v}, \qquad \forall  w,v \in \rset^n, \xi \in \Omega.
\end{align*}
and $f$ is $\sigma_f-$strongly convex,	i.e. there exists $\sigma_f \ge 0$ satisfying:
\begin{align}
f(w) \ge f(v) + \langle \nabla f(v), w-v\rangle + \frac{\sigma_f}{2}\norm{w-v}^2 \qquad \forall w,v \in \rset^n.
\end{align}

\noindent $(ii)$ There exists subgradient mapping $g_h: \rset^n \times \Omega \mapsto \rset^n$ such that $g_h(w;\xi) \in \partial h(w;\xi)$ and $\mathbb{E}[g_h(w;\xi)] \in \partial h(w).$ 

\noindent $(iii)$  $h(\cdot;\xi)$ has bounded gradients on the optimal set: there exists $\mathcal{S}  < \infty $ such that $\mathbb{E}\left[\norm{g_h(w^*;\xi)}^2\right]  \le \mathcal{S}$ for all $w^* \in W^*$;
\end{assumption}
\noindent Condition $(i)$ of the above assumption is natural in composite (stochastic) optimization \cite{Nes:13,BecTeb:09,PatNec:18}. Assumption \ref{assump_basic} condition $(ii)$ guarantees the existence of a subgradient mapping for functions $h(\cdot;\xi)$. Denote $\partial F(w;\xi) = \nabla f(w;\xi) + \partial h(w;\xi)$. 
Moreover, since $0 \in \partial F(w^*)$ for any $w^*\in W^*$, then we assume in the sequel that $g_F(w^*):=\mathbb{E}[g_F(w^*;\xi)] = 0$. Also condition $(iii)$ of Assumption \ref{assump_basic} is standard in the literature related to stochastic algorithms.

%
%
%
%
%
\noindent Given some smoothing parameter $\mu>0$ and $I \subset [m]$, we define the prox operator:
\begin{align*}
\text{prox}_{h,\mu_k}(w;I) = \arg\min\limits_{z \in \rset^n} \frac{1}{|I|}\sum\limits_{i \in I} \; h(z;i) + \frac{1}{2\mu} \norm{z - w}^2 
\end{align*} 
In particular, when $h(w;\xi) = \mathbb{I}_{X_{\xi}}(w)$ the prox operator becomes the projection operator $\text{prox}_{h,\mu}(w;\xi) = \pi_{X_{\xi}}(w)$.
Further we denote $[m] =\{1,\cdots, m\}$. Given the sequence $\mu_k  = \frac{\mu_0}{k}$ then  a useful inequality for the sequel is: \begin{align}\label{mu_bound}
 \sum\limits_{i=0}^{T} \mu_i^{\gamma} \le \mu_0\left( 1+  \frac{T^{1-\gamma}}{1-\gamma} \right) 
\end{align}

\section{Stochastic Proximal Gradient with Minibatches}

\noindent In the following section we present the Stochastic Proximal Gradient with Minibatches (SPG-M) and analyze the complexity of a single iteration under assumption \ref{assump_basic}. 
Let $w^0 \in \rset^n$ be a starting point and $\{\mu_k\}_{k \ge 0}$ be a nonincreasing positive sequence of
stepsizes. 
\begin{flushleft}
\textbf{ Stochastic Proximal Gradient with Minibatches (SPG-M)}: \quad \\
For $k\geq 0$ compute: \\
1. Choose randomly i.i.d. $N-$tuple $I^k \subset \Omega$ w.r.t. probability distribution $\mathbb{P}$\\
2. Update: 
\begin{align*}
 v^k &= w^k - \frac{\mu_k}{N}\sum\limits_{i \in I^k} \nabla f(w^k;i) \\
w^{k+1} &= \arg\min\limits_{z \in \rset^n} \frac{1}{N}\sum\limits_{i \in I^k} \; h(z;i) + \frac{1}{2\mu} \norm{z - v^k}^2 
\end{align*} 
3. If the stoppping criterion holds, then \textbf{STOP}, otherwise $k = k+1$.
\end{flushleft}

For computing $v^k$ are necessary an effort equivalent with $N$ vanilla SGD iterations. However, to obtain $w^{k+1}$, a strongly convex inner problem has to be solved and the linear scaling in $N$ holds only in structured cases. In fact, we consider using specific inner schemes to generate a sufficiently accurate suboptimal solution of the inner subproblem.  We provide more details in next section \ref{sec:com_per_iter}. In the particular scenario when $f=\norm{w}^2_2 $ and $h$ represent the nonsmooth prediction loss, then SPG-M learns completely, at each iteration $k$, a predictor $w^{k+1}$ for the minibatch of data samples $I^k$, while maintaining a small distance from the previous predictor. 
 
\noindent For $N = 1$, the SPG-M iteration reduces to $w^{k+1} = \text{prox}_{h,\mu_k} \left(w^k - \mu_k \nabla f(w^k;\xi_k);\xi_k \right)$ being mainly a Stochastic Proximal Gradient iteration based on stochastic proximal maps \cite{SalBia:17}. 

The asymptotic convergence of vanishing stepsize non-minibatch SPG (a single sample per iteration) has been analyzed in \cite{SalBia:17} with application to trend filtering. Moreover, sublinear $\mathcal{O}(1/k)$ convergence rate for non-minibatch SPG has been provided in \cite{PatIro:20}. However, deriving sample complexity for SPG-M with arbitrary minibatches is not trivial since it requires proper estimation of computational effort required by a single iteration.
In the smooth case ($h = 0$), SPG-M reduces to vanilla minibatch SGD \cite{MouBac:11}:
$$w^{k+1}  =  w^k - \frac{\mu_k}{N} \sum\limits_{i \in I^k}\; \nabla f(w^k;i).$$
On the other hand, for nonsmooth objective functions, when $f = 0$, SPG-M is equivalent with a minibatch variant of SPP  analyzed \cite{PatNec:18, AsiDuc:18,TouTra:16,WanWan17_minibatch,WanSre:19}: 
$$w^{k+1}  =  \text{prox}_{h,\mu_k}(w^{k};I^k).$$ 
Next we analyze the computational (sample) complexity of SPG-M iteration and suggest concrete situations when minibatch size $N > 0$ is advantageous over single sample $N=1$ scheme.



\subsection{Complexity per iteration} \label{sec:com_per_iter}

In this section we estimate bounds on the sample complexity $T_v(N)$ of computing $v^{k}$  and $T_w(N)$ for computing $w^{k+1}$. 
Let $I \subset [m], |I| = N$, then it is 
obvious that sample complexity of computing $v^k$ increase linearly with $N$, thus $$T_v(N) = \mathcal{O}(N).$$ 
A more attentive analysis is needed for the proximal step:
\begin{align}\label{prox_step}
 \arg\min_{z \in \rset^n}\;\; \frac{1}{N}\sum\limits_{i \in I} h(z;\xi_i) + \frac{1}{2\mu}\norm{z - w}^2.
\end{align}
Even for small $N > 0$ the solution of the above problem do not have a closed-form and certain auxiliary iterative algorithm must be used to obtain an approximation of the optimal solution.
For the above primal form, the stochastic variance-reduction schemes are typically called to approach this finite-sum minimization, when $h$ obey certain smoothness assumptions. However, up to our knowledge, variance-reduction methods are limited for our general convex nonsmooth regularizers $h(\cdot;\xi)$. SGD attains an $\delta-$suboptimal point $\norm{\tilde{z} - \text{prox}_{\mu}(w;I)}^2 \le \delta$ at a sublinear rate, in $\mathcal{O}\left(\frac{\mu}{\delta}\right)$ iterations. This sample complexity is independent of $N$ but to obtain high accuracy a large number of iterations have to be performed.

The following dual form is more natural:
\begin{align}
& \min_{z \in \rset^n} \;\; \frac{1}{N}\sum\limits_{i \in I} h(z;\xi_i) + \frac{1}{2\mu}\norm{z - w}^2 \nonumber \\
& =\min_{z \in \rset^n}\;\; \frac{1}{N}\sum\limits_{i \in I} \max\limits_{v_i} \langle v_i,z\rangle - h^*(v_i;\xi_i) + \frac{1}{2\mu}\norm{z - w}^2 \nonumber \\
& = \max\limits_{v}  \min_{z \in \rset^n}\;\; \left \langle  \frac{1}{N}\sum\limits_{i \in I} v_i,z \right\rangle - h^*(v_i;\xi_i) + \frac{1}{2\mu}\norm{z - w}^2 \nonumber \\
& =\max_{v \in \rset^{Nn}}\;\; -\frac{\mu}{2N^2} \norm{\sum_{j=1}^N v_j}^2 + \left\langle  \frac{1}{N}\sum\limits_{j = 1}^N v_i,w \right\rangle - \frac{1}{N}\sum\limits_{j=1}^N \; h^*(v_j;\xi_{i_j}) \label{dual_subproblem}
\end{align}
Moreover, in the interesting particular scenarios when regularizer $h(\cdot;\xi)$ results from the composition of a convex function with a linear operator $h(w;\xi) = l(a_{\xi}^Tw)$, the dual variable reduces from $Nn$ to $N$ dimensions. In this case 
\begin{align}
& \min_{z \in \rset^n} \;\; \frac{1}{N}\sum\limits_{i \in I} l(a_{\xi_i}^T z) + \frac{1}{2\mu}\norm{z - w}^2 \nonumber \\
& = \frac{1}{N}\max_{v \in \rset^{N}}\;\; -\frac{\mu}{2N} \norm{\sum_{j=1}^N a_{\xi_j} v_j}^2 + \left\langle \sum\limits_{j = 1}^N  a_{\xi_i}v_i,w \right\rangle - \sum\limits_{j=1}^N \; l^*(v_j) \label{dual_composition_subproblem}
\end{align}

Computing the dual solution $v^*$, then the primal one is recovered by $z(w) = w - \frac{\mu}{N}\sum\limits_{i=1}^N v_i^*$ for  \eqref{dual_subproblem} or $z(w) = w - \frac{\mu}{N}\sum\limits_{i=1}^N a_{\xi_i} v_i^*$ for \eqref{dual_composition_subproblem}. In the rest of this section we will analyze only the general subproblem \eqref{dual_subproblem}, since the sample complexity estimates will be easily translated to particular instance \eqref{dual_composition_subproblem}.
For a suboptimal $\tilde{v}$ satisfying $\norm{\tilde{v} - v^*} \le \delta$, primal suboptimality with $\delta$ accuracy is obtained by: let $\tilde{z}(w) = w - \frac{\mu}{N}\sum\limits_{i=1}^N \tilde{v}_i$
\begin{align}\label{primal_dual_subopt}
 \norm{\tilde{z}(w) - z(w)} 
 & = \mu \left\|\frac{1}{N} \left(\sum\limits_{i=1}^N \tilde{v}_i -  \sum\limits_{i=1}^N v_i^*\right) \right\| \nonumber \\
 & \le  \frac{\mu}{N} \sum\limits_{i=1}^N \left\|  \tilde{v}_i -  v_i^*  \right\| \le \frac{\mu}{\sqrt{N}} \norm{v^* - \tilde{v}} \le \frac{\mu \delta}{\sqrt{N}}
\end{align}
Further we provide several sample complexity estimations of various dual algorithms to attain primal $\delta-$suboptimality, for general and particular regularizers $h(\cdot;\xi)$.
Notice that the hessian of the smooth component $\frac{\mu}{2N} \norm{\sum_{j =1}^N v_j}^2$ is upper bounded by $\mathcal{O}(\mu)$.
Without any growth properties on $h^*(\cdot;\xi)$, one would be able to solve \eqref{dual_subproblem}, using Dual Fast Gradient schemes with $\mathcal{O}(Nn)$ iteration cost, in $\mathcal{O}\left(\max\left\{Nn, Nn \sqrt{\frac{\mu R_d^2}{\delta}} \right\} \right)$ sample evaluations to get a $\delta$ accurate dual solution. This implies, by \eqref{primal_dual_subopt}, that there are necessary 
$$, \qquad T_w^{in} (N;\delta) = \mathcal{O}\left( \max \left\{ Nn, N^{3/4}n \frac{\mu R_d}{\delta^{1/2}} \right\} \right)$$
sample evaluations to obtain primal $\delta-$suboptimality.
For polyhedral $h^*(\cdot;\xi)$, there are many first-order algorithms that attain linear convergence on the above composite quadratic problem \cite{HsiSi:17}. For instance, the Proximal Gradient algorithm have $\mathcal{O}(Nn)$ arithmetical complexity per iteration and attain a $\delta-$suboptimal dual solution in $\mathcal{O}\left(\frac{L}{\hat{\sigma}(I)}\log\left( \frac{1}{\delta} \right)\right)$, where $L$ is the Lipschitz gradient constant and $\hat{\sigma}(I)$ represents the quadratic growth constant of the dual objective function \eqref{dual_subproblem}. Therefore there are necessary $\mathcal{O}\left(\frac{N \mu}{\hat{\sigma}(I)}\log\left( \frac{1}{\delta} \right)\right)$ sample evaluations for $\delta-$dual suboptimality, and:
$$  T_{w,poly}^{in}(N;\delta) = \mathcal{O}\left(\max \left\{ Nn,\frac{N \mu}{\hat{\sigma}(I)}\log\left( \frac{\mu}{N^{1/2}\delta} \right) \right\}\right)$$ 
sample evaluations to attain primal $\delta-$suboptimal solution.

\section{Iteration complexity in expectation} \label{sec:iteration_com}

\noindent Further, in this section, we estimate the number of SPG-M iterations that is necessary to get an $\epsilon-$suboptimal solution of \eqref{problem_intro}.
We will use the following elementary relation: for any $a,b \in \rset^n$ and $\beta > 0$ we have
\begin{align}
\langle a,b \rangle & \le \frac{1}{2\beta}\norm{a}^2 + \frac{\beta}{2}\norm{b}^2 \label{elem_bound_norm1}\\
\norm{a+b}^2 &\le \left(1 + \frac{1}{\beta}\right)\norm{a}^2 + (1 + \beta)\norm{b}^2.\label{elem_bound_norm2}
\end{align}
The main recurrences which will finally generate our sublinear convergence rates are presented below.
\begin{theorem}\label{th_reccurence}	
Let Assumptions \ref{assump_basic} hold and $\mu_k \le \frac{1}{4L_f}$.  Assume $ \norm{w^{k+1} - \text{prox}_{h,\mu}(v^k;I^k)} \le \delta_k$, then the sequence $\{w^k\}_{k \ge 0}$ generated by SPG-M satisfies:
\begin{align*}
\mathbb{E}[\norm{w^{k+1} - w^*}^2] \le  & \left(1 - \frac{\sigma_f \mu_k}{2} \right) \mathbb{E}[\norm{w^k-w^*}^2]  \\
& \hspace{2cm}  +  \mu_k^2 \frac{\mathbb{E} \left[  \norm{g_F(w^*;\xi)}^2\right]}{N} + \left( 3 + \frac{2}{\sigma_f \mu_k}\right)\delta_k^2 
\end{align*}
\end{theorem}
\begin{proof}
Denote $\bar{w}^{k+1} = \text{prox}_{h,\mu_k}(v^k;I^k)$ and recall that $\frac{1}{\mu_k}\left(v^k - \bar{w}^{k+1} \right) \in \partial h(\bar{w}^{k+1};I^k)$, which implies that there exists a subgradient $g_h(\bar{w}^{k+1};I^k)$ such that 
\begin{align}\label{subproblem_optcond}
g_h(\bar{w}^{k+1};I^k) + \frac{1}{\mu_k}\left(\bar{w}^{k+1} -v^k \right) = 0.
\end{align} 
Using these optimality conditions we have:
\begin{align}
	&\norm{w^{k+1} - w^*}^2 
	= \norm{w^k - w^*}^2  +
	2 \langle w^{k+1} - w^k, w^k-w^* \rangle + \norm{w^{k+1} - w^k}^2 \nonumber\\
	& \overset{\eqref{elem_bound_norm2}}{\le} \norm{w^k - w^*}^2  +
	2 \langle \bar{w}^{k+1} - w^k, w^k-w^* \rangle + 2 \langle w^{k+1} - \bar{w}^{k+1}, w^k-w^* \rangle \nonumber\\
	& \hspace{3cm} + \frac{3}{2}\norm{\bar{w}^{k+1} - w^k}^2 + 3\norm{\bar{w}^{k+1} - w^{k+1}}^2 \nonumber\\
	& = \norm{w^k - w^*}^2  +
	2 \langle \bar{w}^{k+1} - w^k, \bar{w}^{k+1}-w^* \rangle + 2 \langle w^{k+1} - \bar{w}^{k+1}, w^k-w^* \rangle \nonumber\\
	& \hspace{3cm} - \frac{1}{2}\norm{\bar{w}^{k+1} - w^k}^2 + 3\norm{\bar{w}^{k+1} - w^{k+1}}^2 \nonumber\\
	& \overset{\eqref{elem_bound_norm1}}{\le} \left(1 + \frac{\sigma_f\mu_k}{2}\right)\norm{w^k-w^*}^2 +
	2 \langle \mu_k \nabla f(w^k;I_k) + \mu_k g_h(\bar{w}^{k+1};I_k), w^* - \bar{w}^{k+1} \rangle \nonumber\\	& \hspace{3cm} - \frac{1}{2}\norm{\bar{w}^{k+1} - w^k}^2 + \left(3 + \frac{2}{\sigma_f\mu_k}\right)\delta_k^2. \label{start_reccurence}
\end{align}
Now by using convexity of $h$ and Lipschitz continuity of $\nabla f(\cdot;I_k)$, we can further derive:
\begin{align*}
	&2 \mu_k\langle  \nabla f(w^k;I_k) +  g_h(\bar{w}^{k+1};I_k), w^* - \bar{w}^{k+1} \rangle - \frac{1}{2}\norm{\bar{w}^{k+1} - w^k}^2
	\\
	& \le 
	2\mu_k  \langle \nabla f(w^k;I_k), w^* - \bar{w}^{k+1} \rangle - \frac{1}{2}\norm{\bar{w}^{k+1} - w^k}^2 \nonumber\\
	& \hspace{5cm} + 2\mu_k(h(w^* ;I_k) - h(\bar{w}^{k+1};I_k))\\
	& \le  -
	2\mu_k \bigg( \langle \nabla f(w^k;I_k), \bar{w}^{k+1} - w^k \rangle + \frac{1}{8\mu_k}\norm{\bar{w}^{k+1} - w^k}^2 \\
	&  + h(\bar{w}^{k+1};I_k)\bigg)  + 2\mu_k \langle \nabla f(w^k;I_k), w^* - w^{k}\rangle - \frac{1}{4}\norm{\bar{w}^{k+1} - w^k}^2 + 2\mu_k h(w^*;I_k).
	\end{align*}
	By taking expectation w.r.t. $\xi_k$ in both sides, we obtain:
	\begin{align}
	& - 2\mu_k \mathbb{E}
	\bigg[ \langle \nabla f(w^k;I_k), \bar{w}^{k+1} - w^k \rangle + \frac{1}{8\mu_k}\norm{\bar{w}^{k+1} - w^k}^2 + h(\bar{w}^{k+1};I_k)\bigg] \nonumber\\
	&   + 2\mu_k \langle \nabla f(w^k), w^* - w^{k}\rangle - \frac{1}{4}\mathbb{E}\left[\norm{\bar{w}^{k+1} - w^k}^2\right] + 2\mu_k h(w^*)] \nonumber \\
	&\overset{\mu_k < \frac{1}{4L_f}}{\le} 
	 2\mu_k \mathbb{E}[\left( f(w^k;I_k) - F(\bar{w}^{k+1};I_k) \right)] + 2\mu_k \langle \nabla f(w^k), w^* - w^{k}\rangle  \nonumber\\
	& \hspace{3cm} - \frac{1}{4}\mathbb{E}[\norm{\bar{w}^{k+1} - w^k}^2] + 2\mu_k h(w^*) \nonumber \\
	& \le 
	2\mu_k \left(f(w^k) - \mathbb{E}\left[F(\bar{w}^{k+1};I_k) \right] \right) \nonumber \\
	& \hspace{2cm} + 2\mu_k \left( F(w^*) - f(w^k) \!-\! \frac{\sigma_f}{2}\norm{w^k-w^*}^2 \right) \!-\! \frac{1}{4}\norm{\bar{w}^{k+1} - w^k}^2 \nonumber \\
	& \!= - \sigma_f \mu_k\norm{w^k-w^*}^2 +  2\mu_k \mathbb{E}\left[ F(w^*) \!-\! F(\bar{w}^{k+1};I_k) - \frac{1}{8\mu_k} \norm{\bar{w}^{k+1} - w^k}^2 \right]\!\!. \label{rel_reccurence}
	\end{align}
	By combining \eqref{rel_reccurence} with \eqref{start_reccurence} and by taking the expectation with the entire index history we obtain:
	\begin{align}
	 \mathbb{E}[\norm{w^{k+1} - w^*}^2]
	& \le \left(1 - \frac{\sigma_f \mu_k}{2}\right)\mathbb{E}\left[\norm{w^k-w^*}^2\right] \nonumber \\
	& \hspace{0cm} +  2\mu_k \mathbb{E}\left[ F(w^*) - F(\bar{w}^{k+1};I_k) - \frac{1}{8\mu_k} \norm{\bar{w}^{k+1} - w^k}^2 \right].\label{almostend_recurrence} 
	\end{align}
	A last further upper bound on the second term in the right hand side: let $w^* \in W^*$
	\begin{align}
	& \mathbb{E}\left[  F(\bar{w}^{k+1};\xi_k) - F^* + \frac{1}{8\mu_k}\norm{\bar{w}^{k+1} - w^k}^2  \right] \nonumber\\
	& \ge \mathbb{E}\left[\langle g_F(w^*;\xi_k), \bar{w}^{k+1} - w^*  \rangle + \frac{1}{8\mu_k} \norm{\bar{w}^{k+1} - w^k}^2\right] \nonumber\\
	& \ge \mathbb{E}\left[\langle g_F(w^*;\xi_k), w^k - w^* \rangle + \langle g_F(w^*;\xi_k), \bar{w}^{k+1} - w^k \rangle + \frac{1}{8\mu_k} \norm{\bar{w}^{k+1} - w^k}^2\right] \nonumber\\
	& \ge \mathbb{E}\left[\langle g_F(w^*;\xi), w^k - w^* \rangle + \min_{z} \; \langle g_F(w^*;\xi),  z - w^k \rangle + \frac{1}{8\mu_k}\norm{z - w^k}^2\right] \nonumber\\
	& \ge \langle g_F(w^*), w^k - w^* \rangle  - 2\mu_k \mathbb{E} \left[  \norm{g_F(w^*;\xi)}^2\right] = - 2\mu_k \mathbb{E} \left[  \norm{g_F(w^*;\xi)}^2\right], \label{end_recurence}
	\end{align}
	where we recall that we consider $g_F(w^*) = \mathbb{E}\left[ g_F(w^*;\xi) \right] = 0 $. Finally, from \eqref{almostend_recurrence} and \eqref{end_recurence} we obtain our above result.
\end{proof}

\begin{remark}
	Consider deterministic setting $F(\cdot;\xi) = F(\cdot)$ and $\mu_k = \frac{1}{2L_f}$, then SPG-M becomes the proximal gradient algorithm and Theorem \ref{th_reccurence}$(i)$ holds with $ g_F(w^*;\xi) = g_F(w^*) = 0$, implying that $\Sigma = 0$. Thus the well-known iteration complexity estimate  $\mathcal{O}\left(\frac{L_f}{\sigma_f} \log(1/\epsilon) \right)$ \cite{Nes:13,BecTeb:09} of proximal gradient algorithm is recovered up to a constant.
\end{remark}

\begin{theorem}\label{th_iteration_complexity}
Let assumptions of Theorem \ref{th_reccurence} hold. Also let $\delta_k = \mu_k^{3/2}/N^{1/2}$. Then the sequence $\{w^k\}_{k\ge 0}$ generated by SPG-M attains $ \mathbb{E}[\norm{w^T-w^*}^2] \le \epsilon$ within the following number of iterations: 
\begin{enumerate}
	\item[] [\textbf{Constant stepsize:}] \; Let $K > \mathcal{O}\left( \frac{4\mu_0}{\sigma_f^2 N}\right)$ and $ \mu_k := \frac{2\mu_0}{K} \in \left(0, \frac{1}{4L_f} \right]$, then 
	$$ T \le T^{out}_c := \mathcal{O}\left(\max\left\{\frac{\max\{\Sigma^2, 2/\sigma_f\} \log (2r_0^2 /\epsilon)}{\epsilon \sigma_f^2 N}, 
	\sqrt{\frac{ 72 \log^2 (2r_0^2 /\epsilon)}{\epsilon \sigma_f^3 N}} \right\}\right)$$
	\item[][\textbf{Nonincreasing stepsize:}] \; For $\mu_k = \frac{2\mu_0}{k}$, then
	\begin{align*}
	T \le T^{out}_v = \mathcal{O}\left(\max\left\{\frac{\norm{w^0-w^*}^2}{\epsilon},\frac{\Sigma^2+\mu_0 + 1/\sigma_f}{N\epsilon}  \right\} \right)
	\end{align*}
	\item[] [\textbf{Mixed stepsize:}] \; Let $\mu_k = \begin{cases} \frac{\mu_0}{4L_f} & k < T_1 \\ 
	\frac{2\mu_0}{k}, & T_1 \le k \le T_2 \end{cases}$, then 
	\begin{align*}
	T \le T^{out}_m: =  \underbrace{\frac{4L_f}{\mu_0\sigma_f}\log\left( \frac{2\norm{w^0-w^*}^2}{\epsilon}\right)}_{T_1} +
	\underbrace{\mathcal{O}\left(
	\frac{C}{\epsilon N} \right)}_{T_2},	
	\end{align*}
	where $C = \frac{\mu_0\Sigma^2L_f\sigma_f + \mu_0^2\sigma_f + \mu_0L_f}{L_f^2\sigma_f^2} + \Sigma^2+\mu_0 + 1/\sigma_f.$ 
\end{enumerate}	
\end{theorem}
\begin{proof}
\noindent \textbf{Constant step-size}. Interesting convergence rates arise for proper constant stepsize $\mu_k$. Let $\mu_k := \mu \in \left(0, \frac{1}{4L_f} \right]$ and $\delta_k = \mu^{3/2}/N^{1/2}$, then Theorem \ref{th_reccurence} states that 
\begin{align}
\mathbb{E}[\norm{w^k-w^*}^2] 
& \le \left(1 - \frac{\sigma_f\mu}{2} \right)^k r_0^2 + \frac{1 - (1-\sigma_f\mu/2)^{k}}{\sigma_f\mu/2}\frac{\mu^2}{N}\left(\Sigma^2 + 3\mu + \frac{2}{\sigma_f} \right) \nonumber\\
& \le \left(1 - \frac{\sigma_f\mu}{2} \right)^k r_0^2 + \frac{2\mu}{\sigma_f N}\left(\Sigma^2 + 3\mu + \frac{2}{\sigma_f} \right), \label{conststep_linconv}
\end{align}
which imply a linear decrease of initial residual and, at the same time, the linear convergence of $w^k$ towards a optimum neighborhood of radius $\frac{2\mu}{\sigma_f N}\left(\Sigma^2 + 3\mu + \frac{2}{\sigma_f} \right)$. The radius decrease linearly with the minibatch size $N$. With a more careful choice of constant $\mu$ we can same decrease in the SPG-M convergence rate.
Given $K > 0$, let $\mu = \frac{2\mu_0}{K}$ then after 
\begin{align}\label{linear_resreduction}
 T = \frac{K}{\sigma_f \mu_0}\log\left( \frac{2r_0^2}{\epsilon}\right)
\end{align}
\eqref{conststep_linconv} leads to
\begin{align}
\mathbb{E}[\norm{w^T-w^*}^2] 
& \le \frac{2\mu_0}{K\sigma_f N}\left(\Sigma^2 + \frac{6\mu_0}{K} + \frac{2}{\sigma_f} \right) + \frac{\epsilon}{2} \label{no_opterror}\\
& \le \frac{2\log (2r_0^2 /\epsilon)}{T \sigma_f^2 N}\left(\Sigma^2 + \frac{6\log (2r_0^2 /\epsilon)}{\sigma_f T} + \frac{2}{\sigma_f} \right) + \frac{\epsilon}{2}, \nonumber
\end{align}
Overall, to obtain $\mathbb{E}[\norm{w^T-w^*}^2] \le \epsilon,$ SPG-M has to perform $$ \mathcal{O}\left(\max\left\{\frac{\max\{\Sigma^2, 2/\sigma_f\} \log (2r_0^2 /\epsilon)}{\epsilon \sigma_f^2 N}, 
\sqrt{\frac{ 72 \log^2 (2r_0^2 /\epsilon)}{\epsilon \sigma_f^3 N}} \right\}\right)$$
SPG-M iterations.

\vspace{5pt}

\noindent \textbf{Variable stepsize}. Now let $\mu_k = \frac{2\mu_0}{k}, \delta_k = \frac{\mu_k^{3/2}}{N^{1/2}}$, then Theorem \ref{th_reccurence} leads to:
\begin{align}\label{varstep_rate1_pre}
\mathbb{E}[\norm{w^k-w^*}^2] 
& \le \prod\limits_{j = 1}^k   \left(1 - \frac{\sigma_f\mu_j}{2} \right) r_0^2 + \sum\limits_{i = 1}^k \frac{\mu_i^2}{N}\left(\Sigma^2 + 3\mu_i + \frac{2}{\sigma_f} \right) \prod\limits_{j = i+1}^k \left(1 - \frac{\sigma_f\mu_j}{2} \right) 
\end{align}
By using further the same (standard) analysis from \cite{PatIro:20,PatNec:18}, we obtain:
\begin{align}\label{varstep_rate1}
\mathbb{E}[\norm{w^k-w^*}^2] 
& \le \underbrace{\mathcal{O}\left(\frac{r_0^2}{k}\right)}_{\text{optimization error}} +  \underbrace{\mathcal{O}\left(\frac{\Sigma^2+\mu_0 + 1/\sigma_f}{Nk}  \right)}_{\text{sample error}} 
\end{align}
Notice that, for our stepsize choice, the optimization rate $\mathcal{O}(r_0^2/k)$ is optimal (for strongly convex stochastic optimization) and it is not affected by the variation of minibatch size. Intuitively, the stochastic component within optimization model \eqref{problem_intro} is not eliminated by increasing $N$, only a variance reduction is attained. In \cite{WanWan17_prox}, for bounded gradients objective functions, is stated $\mathcal{O}\left( \frac{L^2}{\sigma_f N k} \right)$ convergence rate for Minibatch-Prox Algorithm in average sequence, using classical arguments. However, this rate is based on knowledge of $\sigma_f$, used in the stepsize sequence $\mu_k = \frac{2}{\sigma_f(k-1)}$. Moreover, in the first step the algorithm has to compute $\text{prox}_{h,+\infty}(\cdot;I^0) = \arg\min\limits_{z} \; \frac{1}{N} \sum\limits_{ i \in I}\; h(z;i)$, which for small $\sigma_f$ might be computationally expensive. Notice that, under knowledge  of $\sigma_f$, we could obtain similar sublinear rate in $\mathbb{E}[\norm{w^k-w^*}^2]$ using the same stepsize sequence. 
Returning to \eqref{varstep_rate1}, it implies the following iteration complexity:
\begin{align*}
 \mathcal{O}\left(\max\left\{\frac{r_0^2}{\epsilon},\frac{\Sigma^2+\mu_0 + 1/\sigma_f}{N\epsilon}  \right\} \right).
\end{align*}

\vspace{5pt}

\noindent \textbf{Mixed stepsize.}  By combining constant and variable stepsize policies, we aim to get a better "optimization error" and overall, a better iteration complexity of SPG-M. Inspired by \eqref{linear_resreduction}-\eqref{no_opterror}, we are able to use a constant stepsize policy to bring $w^k$ in a small neighborhood of $w^*$  
whose radius is inversely proportional with $N$.

Let $\mu_k = \frac{\mu_0}{4L_f}$, using similar arguments as in \eqref{linear_resreduction}-\eqref{no_opterror}, we have that after:
\begin{align*}
T_1 \ge \frac{4L_f}{\mu_0\sigma_f}\log\left( \frac{2r_0^2}{\epsilon}\right)
\end{align*}
the expected residual is bounded by:
\begin{align}
\mathbb{E}[\norm{w^{T_1}-w^*}^2]
& \le \frac{\mu_0}{4L_f\sigma_f N}\left(\Sigma^2 + \frac{3\mu_0}{4L_f} + \frac{2}{\sigma_f} \right) + \frac{\epsilon}{2}. 
\end{align}
Now restarting SPG-M from $w^{T_1}$ we have from \eqref{varstep_rate1} that:
\begin{align}
\mathbb{E}&[\norm{w^k-w^*}^2] 
 \le \mathcal{O}\left(\frac{\mathbb{E}[\norm{w^{T_1}-w^*}^2]}{k}\right) +  \mathcal{O}\left(\frac{\Sigma^2+\mu_0 + 1/\sigma_f}{Nk}  \right)\nonumber \\
& \le \mathcal{O}\left(\frac{\mu_0\Sigma^2L_f\sigma_f + \mu_0^2\sigma_f + \mu_0L_f}{L_f^2\sigma_f^2 kN}\right) + \mathcal{O}\left(\frac{\epsilon}{2k}\right) +  \mathcal{O}\left(\frac{\Sigma^2+\mu_0 + 1/\sigma_f}{Nk}  \right). \label{mixed_step_convrate}
\end{align}
iterations. Overall, SPG-M computes $w^{T_2}$ such that $\mathbb{E}[\norm{w^{T_2}-w^*}^2] \le \epsilon$ within a number of iterations bounded by
\begin{align*}
T_1 + T_2 = \frac{4L_f}{\mu_0\sigma_f}\log\left( \frac{2r_0^2}{\epsilon}\right)
 + \mathcal{O}\left(
 \frac{C}{\epsilon N} \right),	
\end{align*}
where $C = \frac{\mu_0\Sigma^2L_f\sigma_f + \mu_0^2\sigma_f + \mu_0L_f}{L_f^2\sigma_f^2} + \Sigma^2+\mu_0 + 1/\sigma_f.$
\end{proof}
 We make a few observations about the $T^{out}_m$. For a small conditioning number $\frac{L_f}{\sigma_f}$ the constant stage performs few iterations and the total complexity is dominated by $\mathcal{O}\left(
\frac{C}{\epsilon N} \right)$. This  bound (of the same order as \cite{WanWan17_prox}) present some advantages: unknown $\sigma_f$, evaluation in the last iterate and no uniform bounded gradients assumptions. On the other hand, for a sufficiently large $N = \mathcal{O}(1/\epsilon)$ minibatch size and a proper choice of $\mu_0$, one could perform a constant number of SPG-M iterations. In this case, the mixed stepsize convergence rate provides a link between population risk and empirical risk.

\subsection{Total complexity}

In this section we couple the complexity per iteration estimates 
from Section \ref{sec:com_per_iter} and Section \ref{sec:iteration_com} 
and provide upper bounds on the total complexity of SPG-M.

Often the measure of sample complexity is used for stochastic algorithms, which refers to the entire number of data samples that are used during all iterations of a given algoritmic scheme.
In our case, given the minibatch size $N$ and the total outer SPG-M iterations $T^{out}$, the sample complexity is given by $ NT^{out}$. In the best case $NT^{out}$ is upper bounded by $\mathcal{O}(1/\epsilon)$. 
We consider the dependency on minibatch size $N$ and accuracy $\epsilon$ of highly importance, thus we will present below simplified upper bounds of our estimates. 
\noindent  In Section \ref{sec:com_per_iter}, we analyzed the complexity of a single SPG-M iteration for convex components $h(\cdot;\xi)$ denoted by $T^{in}_v + T^{in}_w$.
Summing the inner effort $T^{in}_v + T^{in}_w$ over the outer number of iterations provided by Theorem \ref{th_iteration_complexity} leads us to the total computational complexity of SPG-M. 
We further derive the total complexity for SPG-M with mixed stepsize policy and use the same notations as in Theorem \ref{th_iteration_complexity}:
\begin{align*}
T^{total}
& = \sum\limits_{i=0}^{T^{out}_m} (T^{in}_v(N) + T^{in}_w(N;\delta)) \\
& \le \sum\limits_{i=0}^{T^{out}_m} \left( \mathcal{O}(Nn) + \mathcal{O}\left( \max \left\{ Nn, N^{3/4}n \frac{\mu_i R_d}{\delta_i^{1/2}} \right\} \right) \right) \\
& \le \sum\limits_{i=0}^{T^{out}_m} \left[ \mathcal{O}(Nn) + \mathcal{O}\left( Nn \mu_i^{1/4} R_d   \right) \right] \\
& \overset{\eqref{mu_bound}}{\le} \mathcal{O}(T^{out}_m Nn) + \mathcal{O}\left( Nn (T^{out}_m)^{3/4} R_d   \right) \\
& \le \mathcal{O}\left(\left(\frac{L_f}{\sigma_f}\log\left(1/\epsilon \right) + \frac{C}{N\epsilon} \right) Nn \right) + \mathcal{O}\left( Nn \left(\frac{L_f}{\sigma_f}\log\left(1/\epsilon \right) + \frac{C}{N\epsilon} \right)^{3/4} R_d   \right) 
\end{align*}
Extracting the dominant terms from the right hand side we finnaly obtain:
\begin{align*}
T^{total} \le \mathcal{O}\left(\frac{L_fNn}{\sigma_f}\log\left(1/\epsilon \right) + \frac{Cn}{\epsilon} \right) + \mathcal{O}\left( N^{1/4}n \left(\frac{C}{\epsilon} \right)^{3/4} R_d   \right).
\end{align*}
The first term $\mathcal{O}\left(\left(\frac{L_f}{\sigma_f}\log\left(1/\epsilon \right) + \frac{C}{N\epsilon} \right) Nn \right)$ is the total cost of the minibatch gradient step $v^k$ and is highly dependent on the conditioning number $\frac{L_f}{\sigma_f}$.
The second term is brought by solving the proximal subproblem by Dual Fast Gradient method
depicting a stronger dependence on $N$ and $\epsilon$ than the first.
Although the complexity order is $\mathcal{O}\left( \frac{Cn}{\epsilon}\right)$, comparable with the optimal performance of single-sample stochastic schemes (with $N = 1$), the above estimate paves the way towards the acceleration techniques based on distributed stochastic iterations. Reduction of the complexity per iteration to $\frac{1}{\tau}(T^{in}_v(N) + T^{in}_w(N;\delta))$, using multiple machines/processors, would guarantee direct improvements in the optimal number of iterations $\mathcal{O}(\frac{Cn}{\tau\epsilon})$. The superiority of distributed variants of SGD schemes for smooth optimization are clear (see \cite{NiuRec:11}), but our results set up the theoretical foundations for distributing the algorithms for the class of proximal gradient algorithms.
 
%
%
%
%
%
%
%
%
%


\section{Numerical simulations}

\subsection{Large Scale Support Vector Machine}

To validate the theoretical implications of the previous sections, we first chose a well known convex optimization problem  in the machine learning field: optimizing a binary Support Vector Machine (SVM) application. To test several metrics and methods, a spam-detection application is chosen using the dataset from \cite{spamDetectionDataset}. The dataset contains about 32000 emails that are either classified as spam or non-spam. This was split in our evaluation in a classical $80\%$ for training and $20\%$ for testing. To build the feature space in a numerical understandable way, three preprocessing steps were made: 

\begin{enumerate}
\item A dictionary of all possible words in the entire dataset and how many times each occured is constructed.
\item The top $200$($=n$ ; the number of features used in our classifier) most-used words indices from the dictionary are stored.
\item Each email entry $i$ then counts how many of each of the $n$ words are in the $i-th$ email's text. Thus, if $X_{i}$ is the numerical entry characteristic to email ${i}$, then $X_{ij}$ contains how many words with index $j$ in the top most used words are in the email's text.
\end{enumerate}

The pseudocode for optimization process is shown below:

\begin{flushleft}
	\quad For $k\geq 0$ compute \\
	1. Choose randomly i.i.d. $N-$tuple $I^k \subset \Omega$ \\
	2. Update: 
	\begin{align*}
	v^{k}  & =  \left(1 - \lambda\mu_k \right) w^k  \\
	u^{k}  & = \arg\max_{u \in [0,1]} \;\;  -\frac{\mu}{2N}\norm{\tilde{X}_{I_k} u}^2_2 +  u^T(e - \tilde{X}_{I_k}^Tv^k) \\
	w^{k+1} & = v^k + \frac{\mu_k}{N} X_{I^k} u^k
	\end{align*}
	3. If the stoppping criterion holds, then \textbf{STOP}, otherwise $k = k+1$.
\end{flushleft}

To compute the optimal solution $w^*$, we let running the binary SVM state-of-the-art method for the dataset using SGD hinge-loss \cite{spamDetectionPaper} for a long time, until we get the top accuracy of the model (~$93.2\%$). Considering this as a performance baseline, we compare the results of training process efficiency between the $SPG-M$ model versus $SGD$ with mini-batches. The comparison is made w.r.t. three metrics:

\begin{enumerate}
\item \textbf{Accuracy}: how well does the current set of trained weights performs at classification between spam versus non-spam.
\item \textbf{Loss}: the hinge-loss result on the entire set of data.
\item \textbf{Error} (or optimality measure): computes how far is the current trained set of weights ($w^k$ at any step $k$ in time). from the optimal ones, i.e. $\norm{w^k - w^*}^2$.
\end{enumerate}

 The comparative results between the two methods and each of the metrics defined above are shown in Fig. \ref{fig:acc},  \ref{fig:loss}, \ref{fig:error}. These were obtained by averaging several executions on the same machine, each with a different starting point. Overall, the results show the advantage of SPG-M method over SGD: while both methods will converge to the same optimal results after some time, SPG-M is capable of obtaining better results all the three metrics in a shorter time, regardless of the batch size being used.  One interesting observation can be seen for the SGD-Const method results, when the loss metric tends to perfom better (\ref{fig:loss}). This is because of a highly tuned constant learning rate to get the best possible result. However, this is not a robust way  to use in practice.

\begin{figure}[htp] 
    \centering
    \subfloat[batchsize 32]{%
        \includegraphics[width=0.5\textwidth]{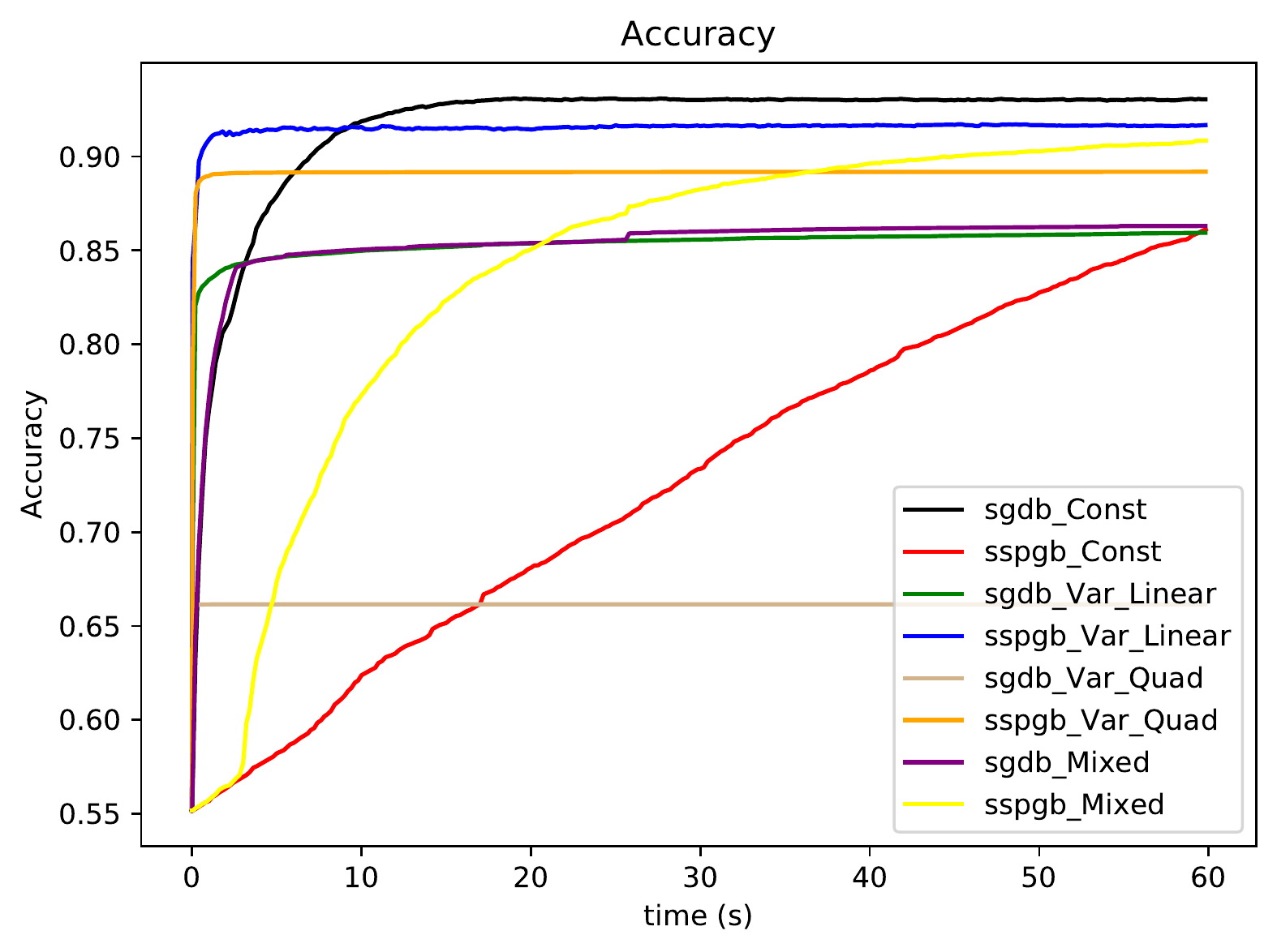}%
        \label{fig:acc32}%
        }%
    \hfill%
    \subfloat[batchsize 128]{%
        \includegraphics[width=0.5\textwidth]{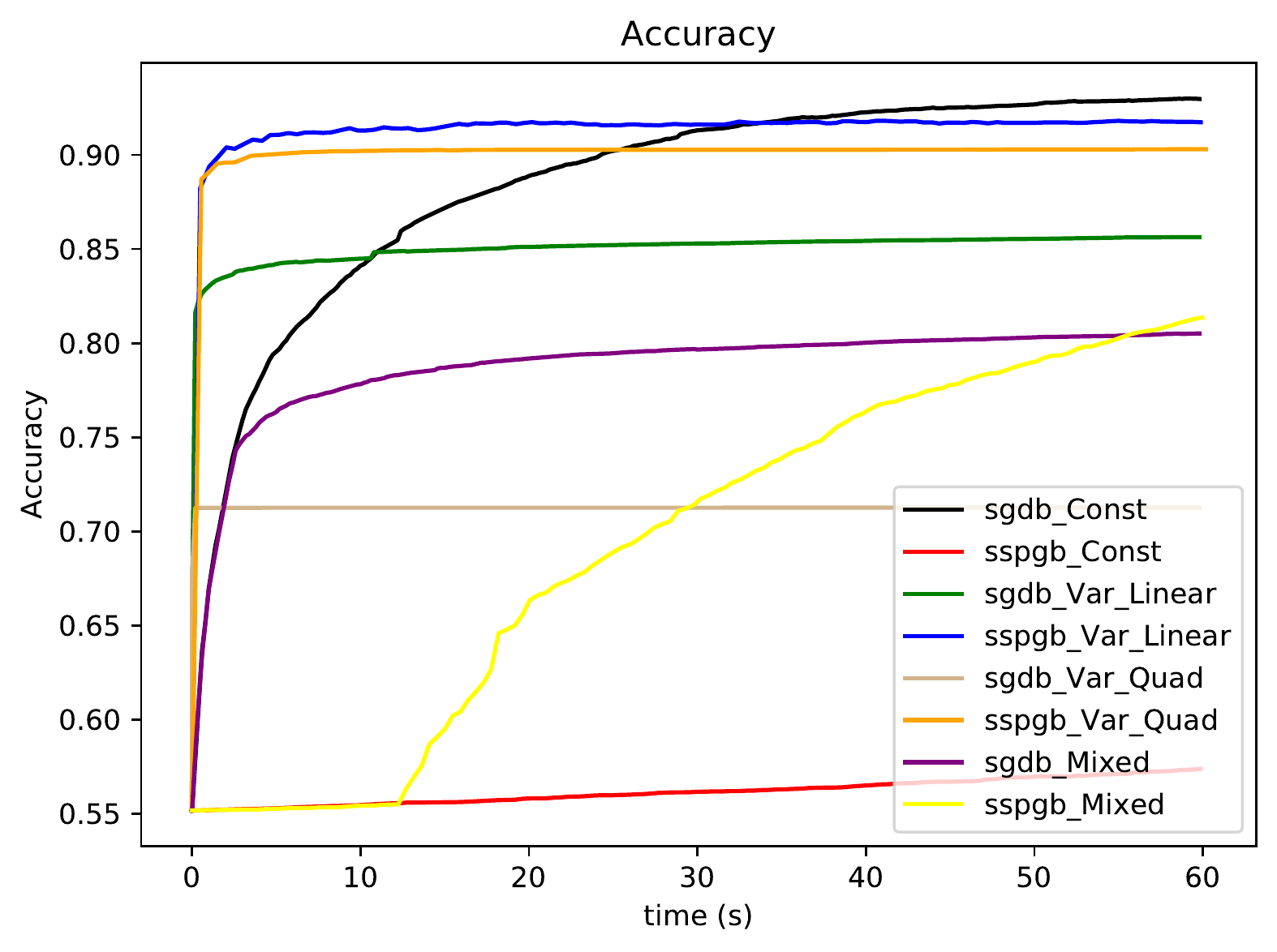}%
        \label{fig:acc128}%
        }%
    \caption{Comparative results between SPG-M and SGD for the \textbf{Accuracy} metric, using different batchsizes and functions for chosing the stepsize.}
   \label{fig:acc}
\end{figure}

\begin{figure}[htp] 
    \centering
    \subfloat[batchsize 32]{%
        \includegraphics[width=0.5\textwidth]{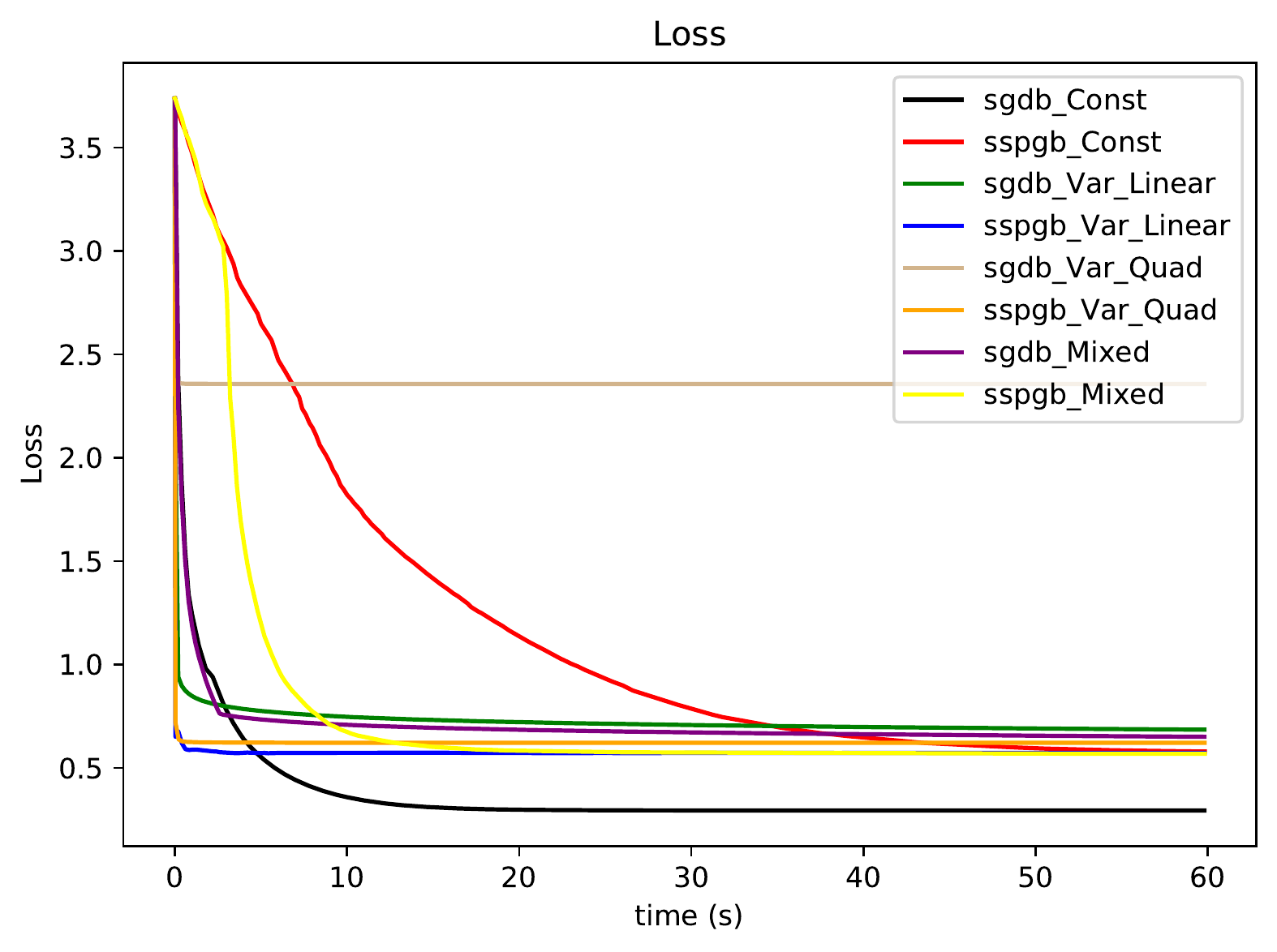}%
        \label{fig:loss32}%
        }%
    \hfill%
    \subfloat[batchsize 128]{%
        \includegraphics[width=0.5\textwidth]{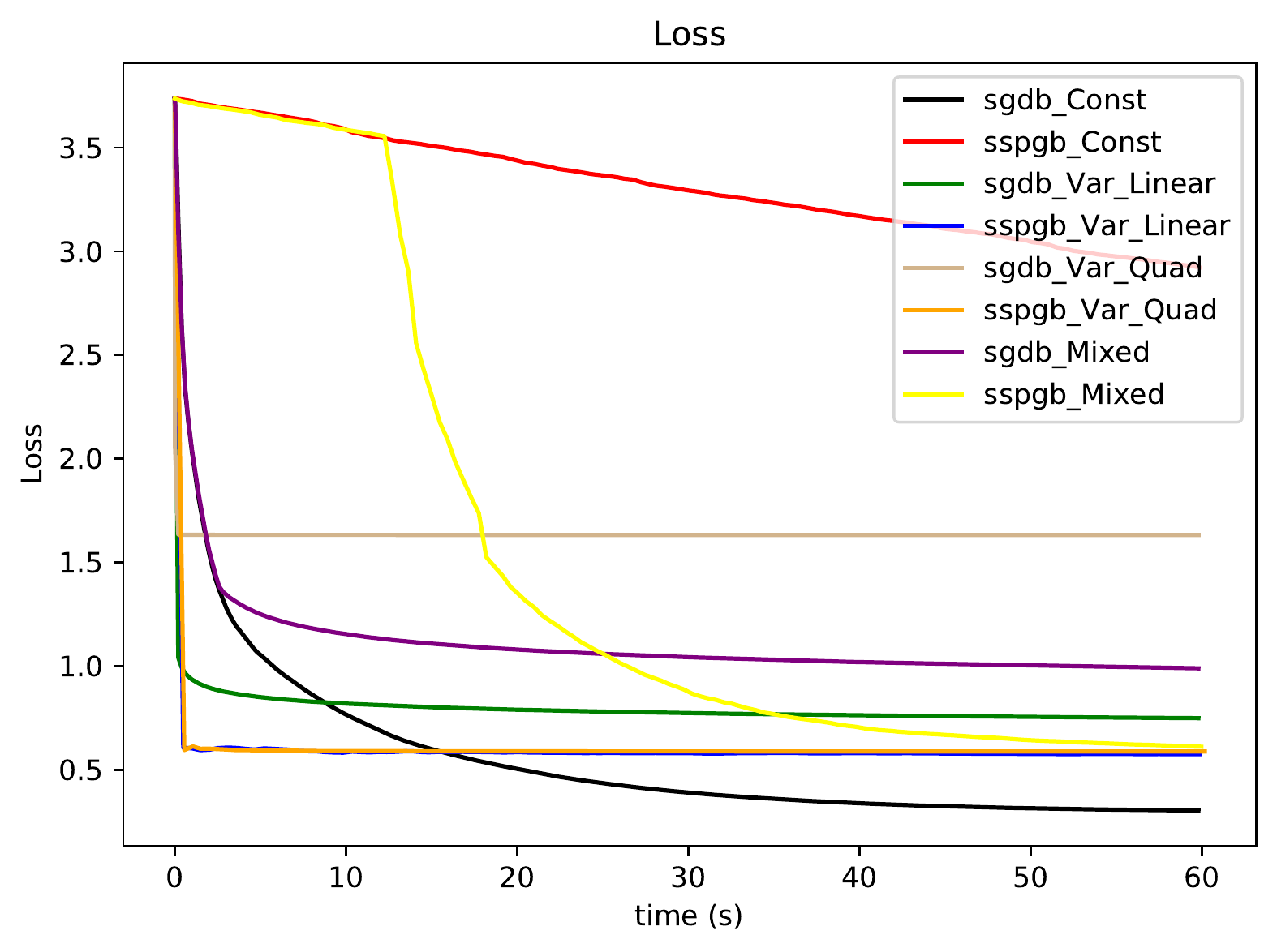}%
        \label{fig:loss128}%
        }%
    \caption{Comparative results between SPG-M and SGD for the \textbf{Loss} metric, using different batchsizes and functions for chosing the stepsize.}
    \label{fig:loss}
\end{figure}

\begin{figure}[htp] 
    \centering
    \subfloat[batchsize 32]{%
        \includegraphics[width=0.5\textwidth]{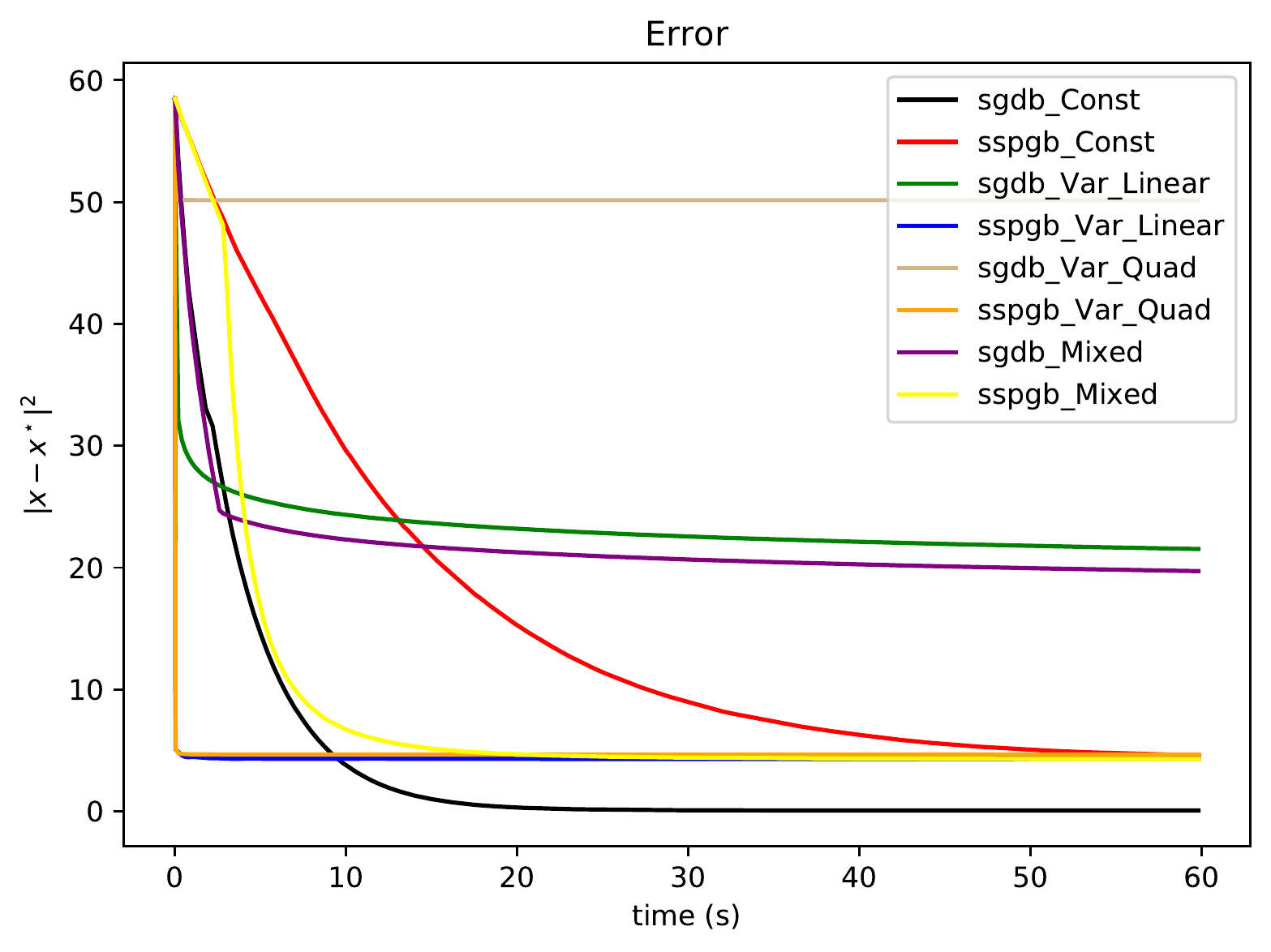}%
        \label{fig:error32}%
        }%
    \hfill%
    \subfloat[batchsize 128]{%
        \includegraphics[width=0.5\textwidth]{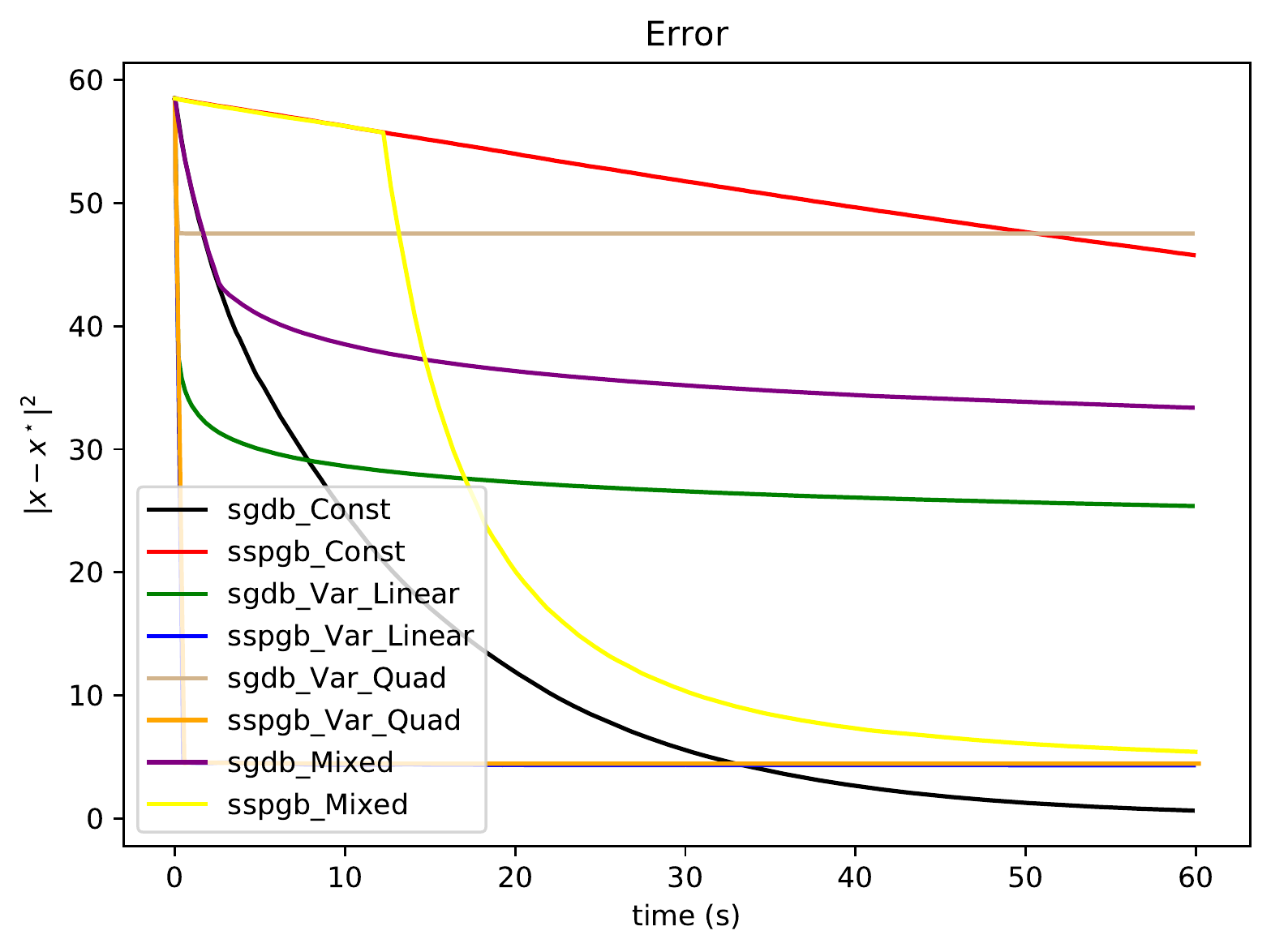}%
        \label{fig:error128}%
        }%
    \caption{Comparative results between SPG-M and SGD for the \textbf{Error} metric, using different batchsizes and functions for chosing the stepsize.}
    \label{fig:error}
\end{figure}

\newpage
\subsection{Parametric Sparse Representation}

\input{parametric}
\section{Conclusion}

\noindent In this chapter we presented preliminary guarantees for minibathc stochastic proximal gradient schemes, which extend some well-known schemes in the literature. For future work, would be interesting to analyze the behaviour of SPG-M scheme on nonconvex learning models.

We provided significant improvements in iteration complexity that future work can further reduce using distributed and parallelism techniques, as hinted by the distributed variants of SGD schemes~\cite{NiuRec:11}.


\bibliographystyle{plain}
\bibliography{sppg}



	
\end{document}

%% file: parametric.tex


Given signal $y \in \rset^m$
and
overcomplete dictionary $T \in \rset^{m \times n}$
(whose columns are also called atoms),
sparse representation~\cite{Ela:10}
aims to find the sparse signal $x\in\rset^n$
by projecting $y$ to a much smaller subspace
generated by a subset of the columns from $T$.
Sparse representation is a key ingredient in dictionary learning techniques~\cite{DumIro:18}
and here we focus on the multi-parametric sparse representation model proposed in \cite{StoIro:19}.
Note that this was analyzed in the past in non-minibatch SPG form~\cite{PatIro:20} which we denote with SSPG in the following.
The multi-parametric representation problem is given by:
\be
\begin{aligned}
	& \underset{x}{\min}
	& & \|Tx-y\|^2_2 \\
	& \text{s.t.}
	& & \|\Delta x\|_1 \le \delta,
\end{aligned}
\label{eq:optprob}
\ee
where $T$ and $x$ correspond to the dictionary and, respectively, the resulting sparse representation,
with sparsity being imposed on a scaled subspace $\Delta x$ with $\Delta\in\rset^{p\times n}$.
In pursuit of \eref{problem_intro},
we move to the exact penalty problem $\min_x \frac{1}{2m}\norm{Tx - y}^2_2 + \lambda \norm{\Delta x }_1$. In order to limit the solution norm we further regularize the unconstrained objective using an $\ell_2$ term as follows:
\begin{align*}
	\min\limits_x \;\; \frac{1}{2m}\|Tx-y\|^2_2 +
	\frac{\alpha}{2} \|x\|^2_2 + 
	\frac{\lambda}{p}\|\Delta x\|_1.
\end{align*}
\label{eq:strongprob}
The decomposition which puts the above formulation into model \eqref{problem_intro} consists of:
\be
f(x;\xi) = \frac{1}{2} (T_\xi x-y_\xi)^2_2 + 
\frac{\alpha}{2} \|x\|^2_2
\label{eq:sp_smooth}
\ee
where $T_\xi$ represents line $\xi$ of matrix $T$, and 
\be
h(x; \xi) = \lambda |\Delta_\xi x|.
\label{eq:sp_nonsmooth}
\ee
To compute the SPG-M iteration for the sparse representation problem, we note that 
\begin{align*}
	\text{prox}_{h,\mu}(x;\I) 
	& = \arg\min_z \; \frac{\lambda}{N}\norm{\Delta_I z}_1 + \frac{1}{2\mu}\norm{z-x}^2.
\end{align*}
Equivalently, once we find dual vector
\begin{align*}
	z^{k}  & = \arg\min_{-1 \le z \le 1}  \frac{\mu\lambda^2}{2N^2} \norm{\Delta_I^Tz}^2 - \frac{\lambda}{N}z^T\Delta_I t 
\end{align*}
then we can easily compute $\text{prox}_{h,\mu}(x;I) = x - \frac{\mu \lambda}{N} \Delta_I^T z$.
We are ready to formulate the resulting particular variant of SPG-M.
\begin{flushleft}
	\textbf{SPG-M - Sparse Representation  (SPGM-SR)}: 
	\quad For $k\geq 0$ compute \\
	1. Choose randomly i.i.d. $N-$tuple $I^k \subset \Omega$ \\
	2. Update: 
	\begin{align*}
	y^{k}  & =  \left[ I_n - \frac{\mu_k}{N} \left(T_{I_k}^T T_{I_k} + N\alpha I_n \right)\right] x^k + \frac{\mu_k}{N} T_{I_k}^T y_{I_k}\\
	z^{k}  & = \arg\min_{-1 \le z \le 1}  \frac{\mu\lambda^2}{2N^2} \norm{\Delta_I^Tz}^2 - \frac{\lambda}{N}z^T\Delta_I t \\
	x^{k+1} & = y^k - \frac{\mu_k\lambda}{N} \Delta_I^T z^k
	\end{align*}
	3. If the stoppping criterion holds, then \textbf{STOP}, otherwise $k = k+1$.
\end{flushleft}

\begin{figure}[t]
    \centering
    \begin{minipage}{0.5\textwidth}
	\centering
	\includegraphics[width=\textwidth]{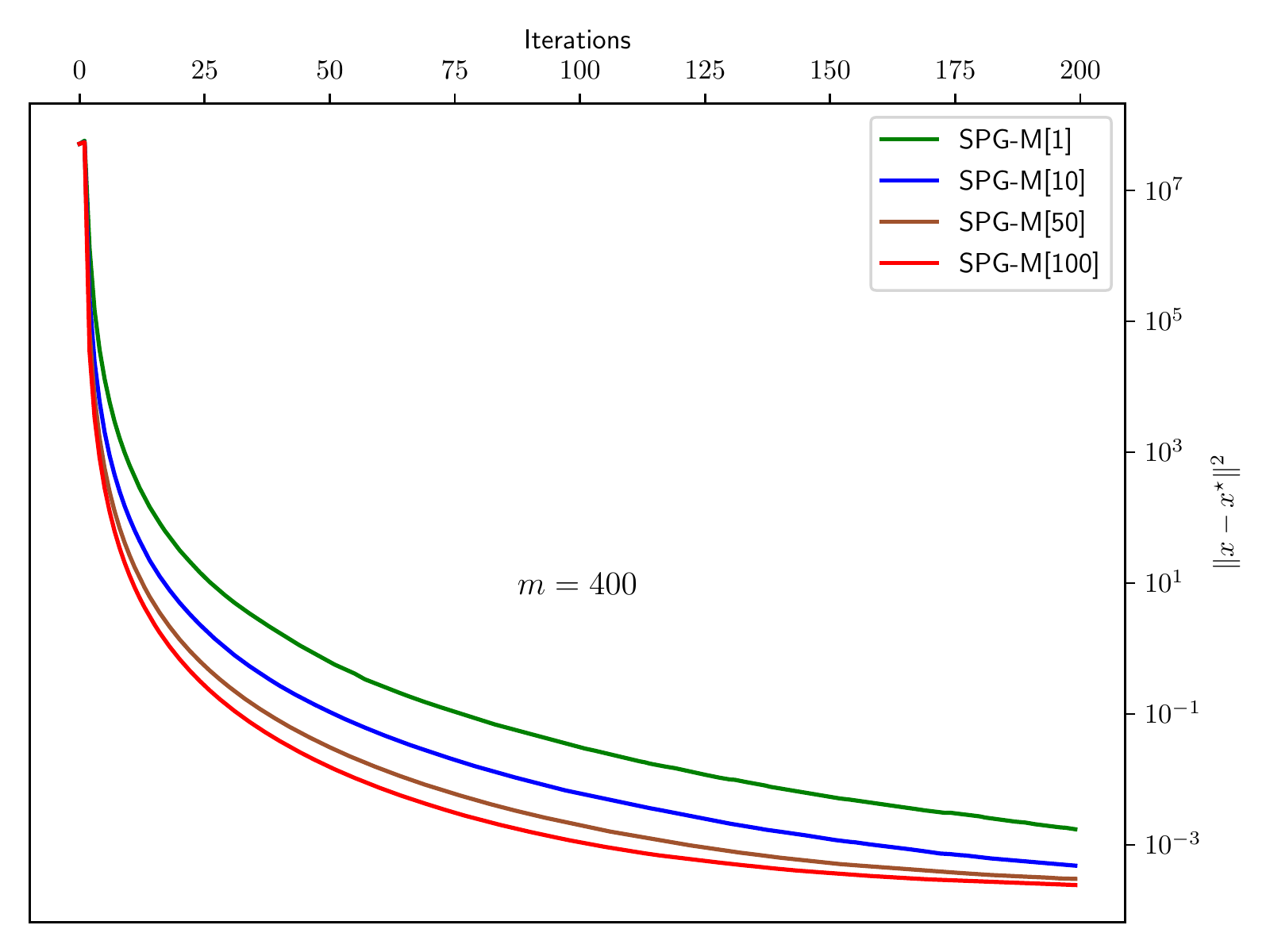}
	\caption{Variable stepsize ($\alpha=0.2$)}
	\label{fig:spgm-gensr-linstep2}
	\end{minipage}\hfill
	\begin{minipage}{0.5\textwidth}
	\centering
	\includegraphics[width=\textwidth]{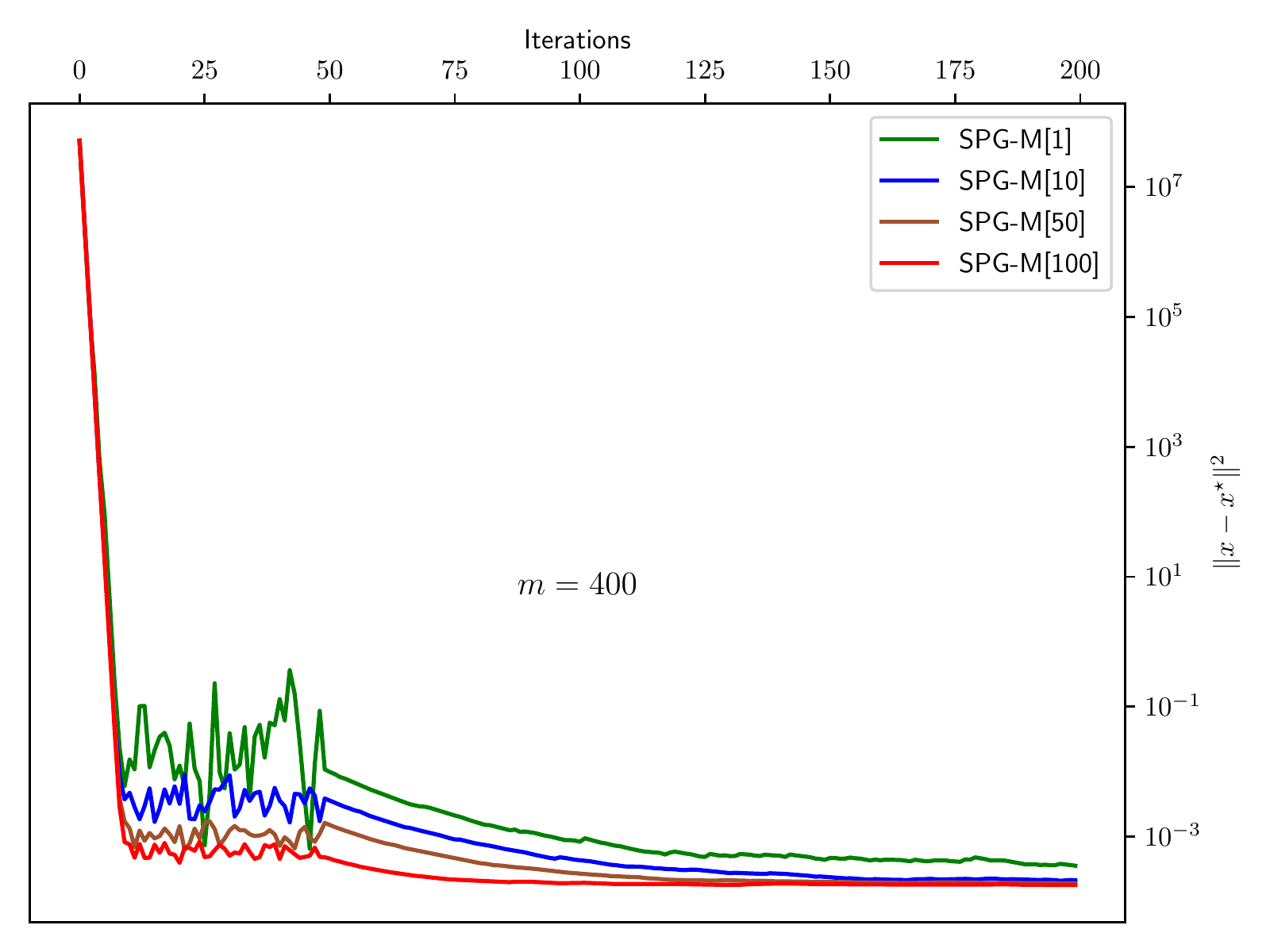}
	\caption{Mixed stepsize ($\alpha=0.2$)}
	\label{fig:spgm-gensr-mixedstep2}
	\end{minipage}
\end{figure}
\begin{figure}[t]
    \centering
    \begin{minipage}{0.5\textwidth}
	\centering
	\includegraphics[width=\textwidth]{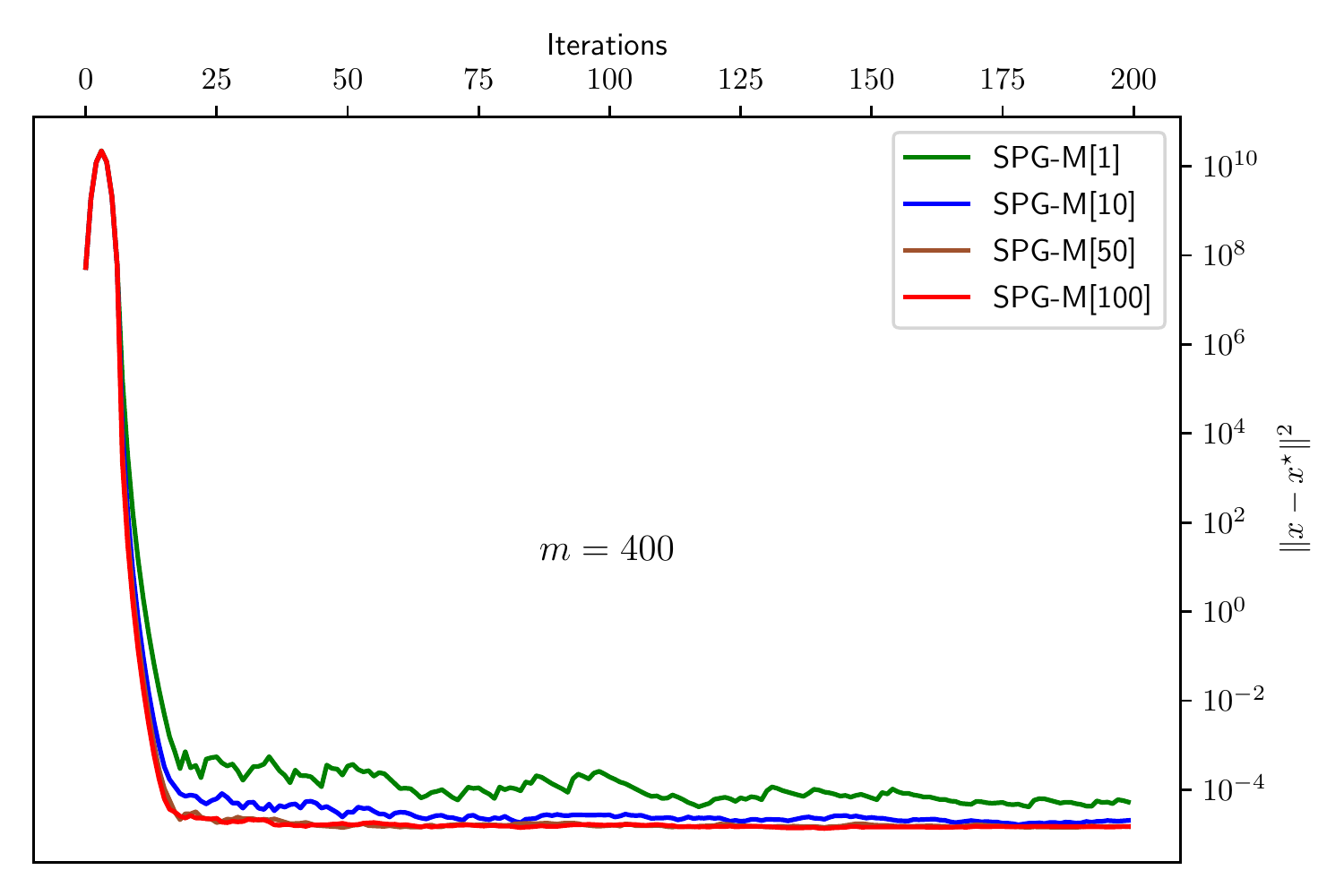}
	\caption{Variable stepsize ($\alpha=0.7$)}
	\label{fig:spgm-gensr-linstep3}
	\end{minipage}\hfill
	\begin{minipage}{0.5\textwidth}
	\centering
	\includegraphics[width=\textwidth]{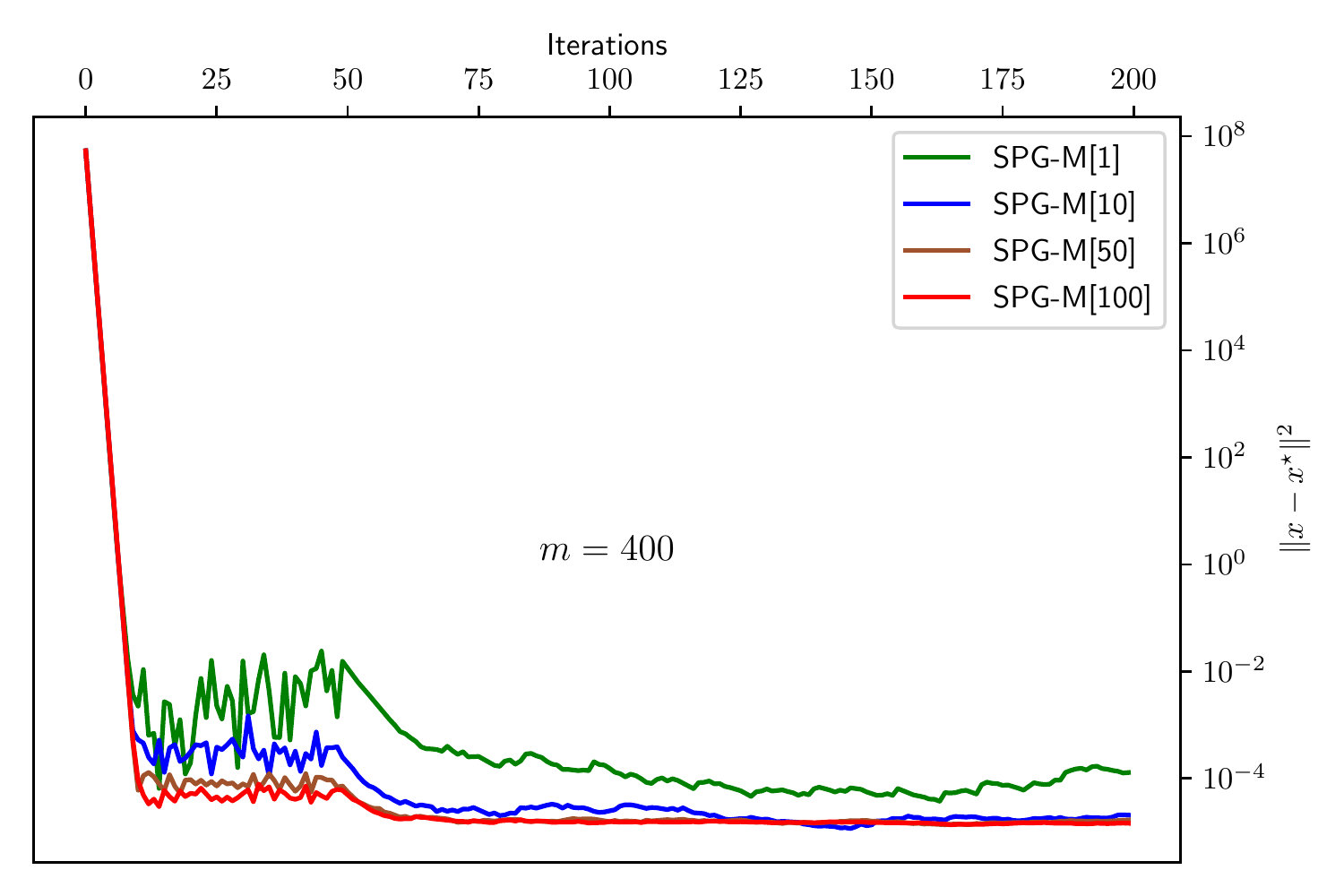}
	\caption{Mixed stepsize ($\alpha=0.7$)}
	\label{fig:spgm-gensr-mixedstep3}
	\end{minipage}
\end{figure}

We proceed with numerical experiments that depict
SPG-M with various batch sizes $N$
and compare them with SSPG~\cite{PatIro:20},
which is equivalent to SPG-M where $N=1$.
In our experiments we use batches of $N=\{1, 10, 50, 100\}$ samples
from a population of $m=400$
using dictionaries with $n=200$ atoms.
We fix $\lambda=5\dot 10^{-4}$ and stop the algorithms
when the euclidean distance between current solution $x$ 
and the optimum $x^{\star}$ is less than $\varepsilon=10^{-3}$.
CVX is used to determine $x^\star$ within a $\varepsilon=10^{-6}$ margin.

Here we choose two scenarios:
one where $\alpha=0.2$ and all methods provide adequate performance,
depicted in Figures \ref{fig:spgm-gensr-linstep2} and \ref{fig:spgm-gensr-mixedstep2},
and a second where $\alpha=0.7$ is larger and stomps performance as can be seen in the first ten iterations
of Figures \ref{fig:spgm-gensr-linstep3} and \ref{fig:spgm-gensr-mixedstep3}.
In these figures we can also observe that the mixed stepsize indeed provides much better convergence rate and that the multibatch algorithm is always ahead of the single batch SSPG version. 

We continue our investigation by adapting the minibatch stochastic gradient descent method to the parametric sparse representation problem which leads to the following algorithm:

\begin{flushleft}
	\textbf{SGD-M - Sparse Representation  (SGDM-SR)}:
	\quad For $k\geq 0$ compute \\
	1. Choose randomly i.i.d. $N-$tuple $I^k \subset \Omega$ \\
	2. Update: 
	\begin{align*}
	x^{k+1} & = x^k - \frac{\mu_k}{N} \left(T_{I_k}^T T_{I_k} + N\alpha I_n \right) x^k + \frac{\mu_k}{N} T_{I_k}^T y_{I_k} - \frac{\mu_k\lambda}{N} \sum\limits_{i \in I^k} \text{sgn}(\Delta_i x^k) \Delta_i^T,
	\end{align*}
	where $\text{sgn}(\Delta_i x) = \begin{cases} +1, & \Delta_i x > 0 \\ -1, & \Delta_i x< 0 \\ 0, & \Delta_i x = 0\end{cases}$
	
	3. If the stoppping criterion holds, then \textbf{STOP}, otherwise $k = k+1$.
\end{flushleft}

When applying SGDM-SR on the same initial data as our experiment with $\alpha=0.7$ with the same parametrization and batch sizes we obtain the results depicted in Figure~\ref{fig:spgm-gensr-sgd}.
Here the non-minibatch version of SGD is clearly less performant than SPG-M,
but what is most interesting is that the minibatch version
for all batch sizes behaves identically and
takes 100 iterations to recover and reach the optimum around iteration 150.
\pdffig{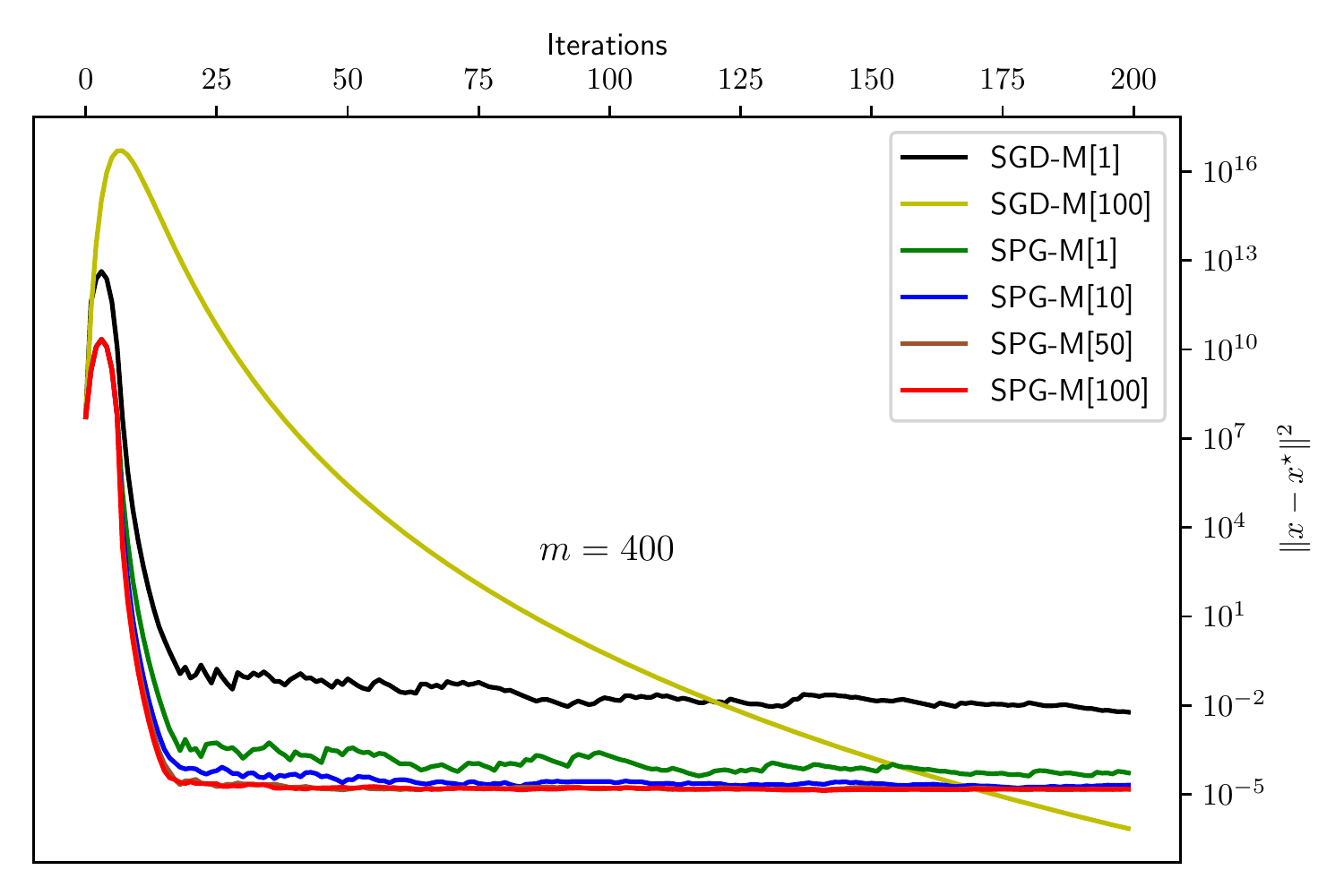}
  {SPG-M and SGD with Variable stepsize. The other batch sizes were identical with N=100 for SGD. ($\alpha=0.7$)}
  {fig:spgm-gensr-sgd}